%% file: main.tex
\documentclass[11pt]{article}

\usepackage{algorithmic}

\usepackage{algorithm}
\usepackage{microtype}
\usepackage{graphicx}
\usepackage{subfigure}
\usepackage{booktabs} % for professional tables
\usepackage{hyperref}
\usepackage{amssymb,amsmath,url,amsthm}
\usepackage{natbib}
\usepackage{wrapfig}
\usepackage[ margin=1in]{geometry}
\theoremstyle{plain}
\newtheorem{theorem}{Theorem}[section]

\newtheorem{lemma}[theorem]{Lemma}

\theoremstyle{definition}

\theoremstyle{remark}

\input{math_commands}
\newcommand{\ours}{\textrm{AutoMAN}}
\newcommand{\bbx}{\bm{x}}
\newcommand{\bbX}{\bm{X}}
\newcommand{\bbF}{\bm{F}}
\newcommand\blfootnote[1]{%
  \begingroup
  \renewcommand\thefootnote{}\footnote{#1}%
  \addtocounter{footnote}{-1}%
  \endgroup
}
\begin{document}
\title{Learned Feature Importance Scores for Automated Feature Engineering}
\date{}
\author{
  Yihe Dong$^\dagger$, 
  Sercan Arik$^\dagger$, 
  Nate Yoder$^\S$, 
  Tomas Pfister$^\dagger$
  }

\maketitle
\begin{abstract}
Feature engineering has demonstrated substantial utility for many machine learning workflows, such as in the small data regime or when distribution shifts are severe. Thus automating this capability can relieve much manual effort and improve model performance. %\yd{update}
Towards this, we propose \ours{}, or \underline{Auto}mated \underline{M}ask-based Fe\underline{a}ture E\underline{n}gineering, an automated feature engineering framework that achieves high accuracy, low latency, and can be extended to heterogeneous and time-varying data. \ours{} is based on effectively exploring the candidate transforms space, \textit{without} explicitly manifesting transformed features. This is achieved by learning feature importance masks, which can be extended to support other modalities such as time series. \ours{} learns feature transform importance end-to-end, incorporating a dataset's task target directly into feature engineering, resulting in state-of-the-art performance with significantly lower latency compared to alternatives. 
\end{abstract}
\blfootnote{$^\dagger$ Google. $^\S$Work done while at Google.}\\

%\keywords{Feature engineering, feature selection, AutoML, tabular, time series.}

\maketitle

\section{Introduction}
Feature engineering is a vital component of a data scientist's toolkit, leading to the success of many models and workflows, such as the winning solultions of notable data science competitions like the Netflix Competition \cite{netflixPrize}. 
Feature engineering can have multitude of benefits.
It has gained importance given the modern data-gathering processes, where the proliferation of data often decreases the signal-to-noise ratio.
It is an effective way of transferring the inductive biases about the target task into the model. It can specifically be useful to improve quality in the small data regime or when severe distribution shifts are present, especially given the challenge that high capacity models like deep neural networks are prone to overfit on noise and spurious correlations, in the underspecification regimes where their trainability and generalization are limited \citep{pmlr-v119-sagawa20a}. %, amour_unders}. 
Feature engineering can be the recipe for obtaining high performance interpretable models that absorb task complexity during feature generation rather than downstream model training. Moreover, it can have product deployment and maintanence benefits via cost-effective feature storage and management systems. Lastly, it can lead to more controllable models for gradual model updates.

%High capacity models like deep neural networks have tendency to overfit on noise and spurious correlations \citep{pmlr-v119-sagawa20a}, thus a well-developed feature engineering tool can provide significant benefits.
%In addition, as real world data types are increasingly multimodal, the ease of the feature engineering tool to adapt to multimodal, such as time series, datasets is an added advantage. 

Automating the process of feature engineering is highly desired, with the motivation of replacing the traditionally time-consuming and expertise-heavy process of crafting powerful features. 
A variety of approaches have been explored towards automated feature engineering, including search-based algorithms that explore the feature space for optimal combinations \citep{openfe, Kanter2015DeepFS, Kaul2017AutoLearnA,shi_2020}, meta-learning approaches that leverage past feature engineering experiences \citep{fetch}, and deep learning models that automatically extract complex features from data \citep{nas}. These have shown promise in real world scenarios, yet challenges remain in obtaining robust performance in a variety of data regimes; mitigating feature explosion caused by complex search spaces; handling high-dimensional and time-varying data; effective scalability to high number of samples and features, while yielding low feature engineering latency and cost; and obtaining interpretability via the discovered features.

%And at a time when real world datasets can scale to millions or billions of samples, hence in addition to accuracy, feature engineering tools need to have high scalability, i.e. low training and inference cost for large datasets, to be applicable in the real world.

In light of these, we propose \ours, a novel automated feature engineering framework that achieves state-of-the-art downstream task performance demonstrated on a variety of benchmarks, while being scalable and adaptable to time series datasets.
\ours{} starts with a pre-determined set of candidate transform functions, then explores the space of all transform-function-feature combinations implicitly by
learning a feature importance mask on the input features for each transform function.
To achieve high accuracy, \ours{} uses the downstream task as the feature engineering objective, thus optimizing feature transforms to directly benefit the task metric \textit{end-to-end}. In particular, \ours{} sidesteps the need to incorporate the task metric into a rewards function as in some prior works \citep{fetch, wang_reinforce_afe}.
To achieve the level of latency required for high scalability, \ours{} rests on the premise that, rather than explicitly manifesting every function in the transform space, as in many prior works, and perform gradient descent on a large number of candidate features, it is sufficient to learn a select set of expressive transform functions with a high potential of introducing meaningful features that improve final task performance. In addition, we demonstration extension of \ours{} to time series data, with the incorporation of learnable temporal masks.
%local masks to select the features for each transform function, and global masks to select across all transformed features.

Our main contributions can be highlighted as:
\begin{itemize}
    \item We propose \ours{}, a simple yet effective feature importance masking based feature engineering approach. \ours{} uses a small set of expert-curated feature transform functions, motivated by theory. %and learns feature importance masks.
    \item \ours{} can be adapted to time series via learnable temporal masks, yielding temporal feature engineering.
    \item The feature engineering complexity of \ours{} scales linearly with the number of features and the number of samples -- thus it can handle large-scale datasets effectively.
    \item \ours{} trains end-to-end with respect to the dataset task objective, and can be easily integrated into any gradient-based model. Experiments show \ours{} notably outperforms baseline automated feature engineering methods across real world datasets.
\end{itemize}

\section{Methods}

\begin{wrapfigure}{R}{10cm}
%\begin{figure}{L}{5cm}
%\vspace{-3pt}
\centering
\includegraphics[width=1.\linewidth]{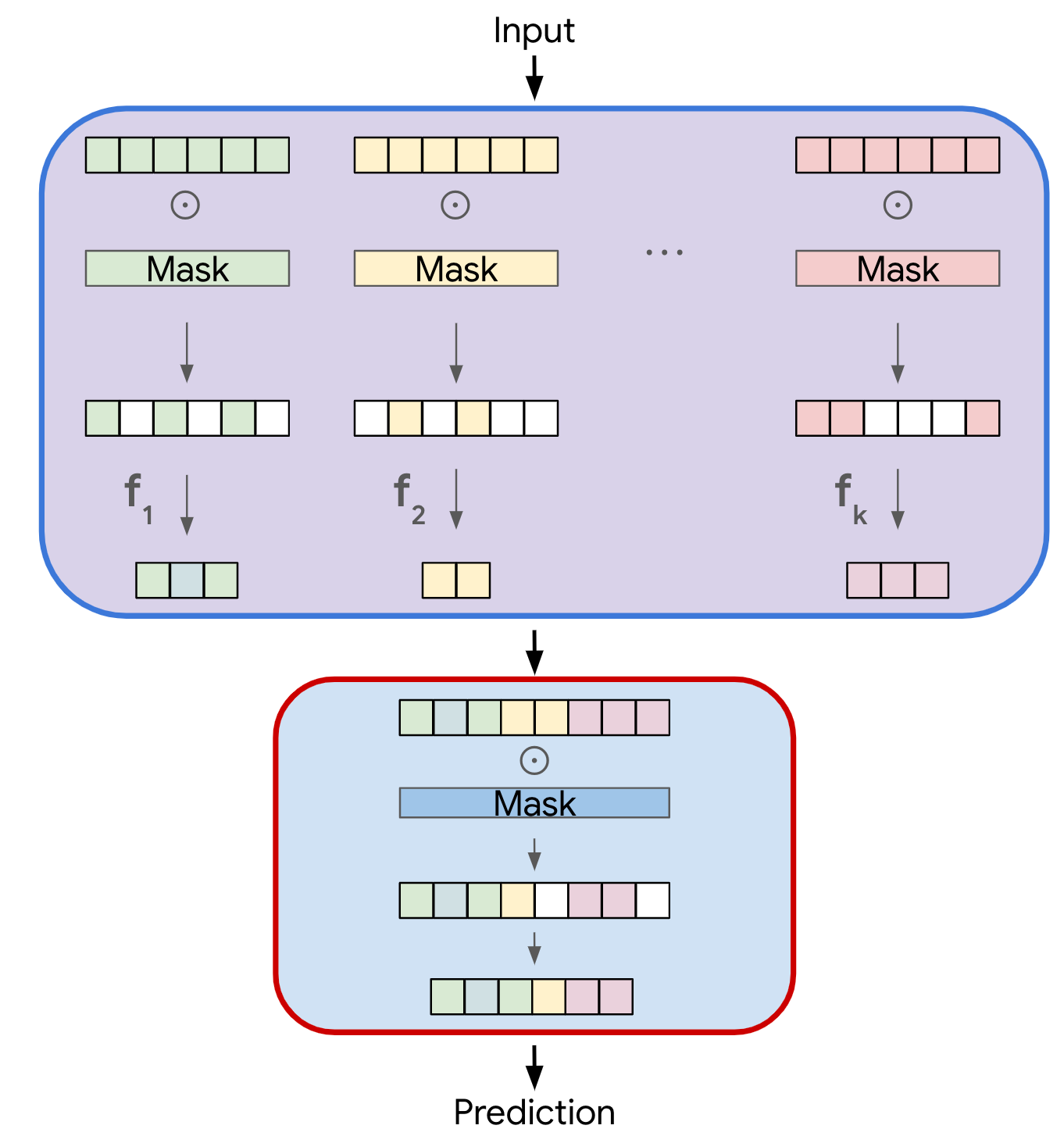}
\caption{Overview of the proposed approach, \ours{}. Local importance masks are learned to select the most relevant features for a given transform, and global importance masks are learned to select the most relevant transformed features across all transforms, all end-to-end using the downstream task objective. The input features are fed into all pertinent transform candidates, depending on whether a feature is categorical, numerical, or temporal. The model prediction head consists of an MLP network. The transform functions $\{f_i\}_{i=1}^k$ are applied to their respective selected features.
}
\label{fig:overall_diagram}
\vspace{-1pt}
\end{wrapfigure}

Fig. \ref{fig:overall_diagram} overviews the proposed \ours{} framework.
\ours{} makes yields effective discovery of engineered features, starting the transform search from a human expert curated set of candidate transform functions, many with learnable components; and learning feature importance masks to explore this search space continuously, without exhaustively manifesting feature transforms. Algorithm~\ref{alg:afe} provides an overview of the \ours{} framework.

\textbf{Notations}. Throughout this work, we let $\mathbf{X} \in \Rbb^{n\times d}$ denote the input data, %$\mathbf{X_{sp}} \in \Rbb^{n\times d}$ the selected features, 
and $\bbm$, $\bbm_{loc}$, and $\bbm_{glb}$ the learned feature importance masks. $\bbF=\{f_i\}_{i=1}^k$ denote the set of $k$ candidate transform functions. 
We use $\odot$ to denote element-wise multiplication between input samples and the mask $\bbm$: $\mathbf{X} = \mathbf{X} \odot \bbm $. 
While there does not exist a single optimal set of transform functions in terms of downstream task performance, we refer to any such set of optimal functions as the target.

\subsection{Efficient search in feature transform space}

\ours{} is based on starting the feature engineering search from a human expert curated set of candidate transform functions. We limit the transform function search space to a small set of functions (shared across the tasks we experiment in this paper). Examples include normalization transform and feature-wise aggregations. To optimize the transform search for a given dataset, many such transforms contain a learnable component. For instance, for a normalization transform that divides each sample by a scaler and subtracts a mean, both the scaler and the mean are learned by the model.
\S\ref{sec:transforms} contains details on the set of transform functions explored. 

Note that while this pre-determined transform set constrains the transform space explored, it greatly improves scalability, as the time complexity is often exponential with respect to the number of transforms explored in standard feature discovery, where a common assumption is that every transform function can be applied to all features, and transform functions can be arbitrarily composed with each other.
In practice, we found that a judiciously-chosen set of general transformations suffices for the myriad of real world datasets tested, consistent with the theory in \S\ref{sec:expressiveness} that a small family of transform functions can well-approximate arbitrary continuous transform functions. That being said, \ours{} can take advantage of more candidates efficiently as it has linear dependence on the number of transform functions (see \S\ref{sec:complexity}). 

\begin{figure}[h]
\centering
\includegraphics[width=0.35\linewidth]{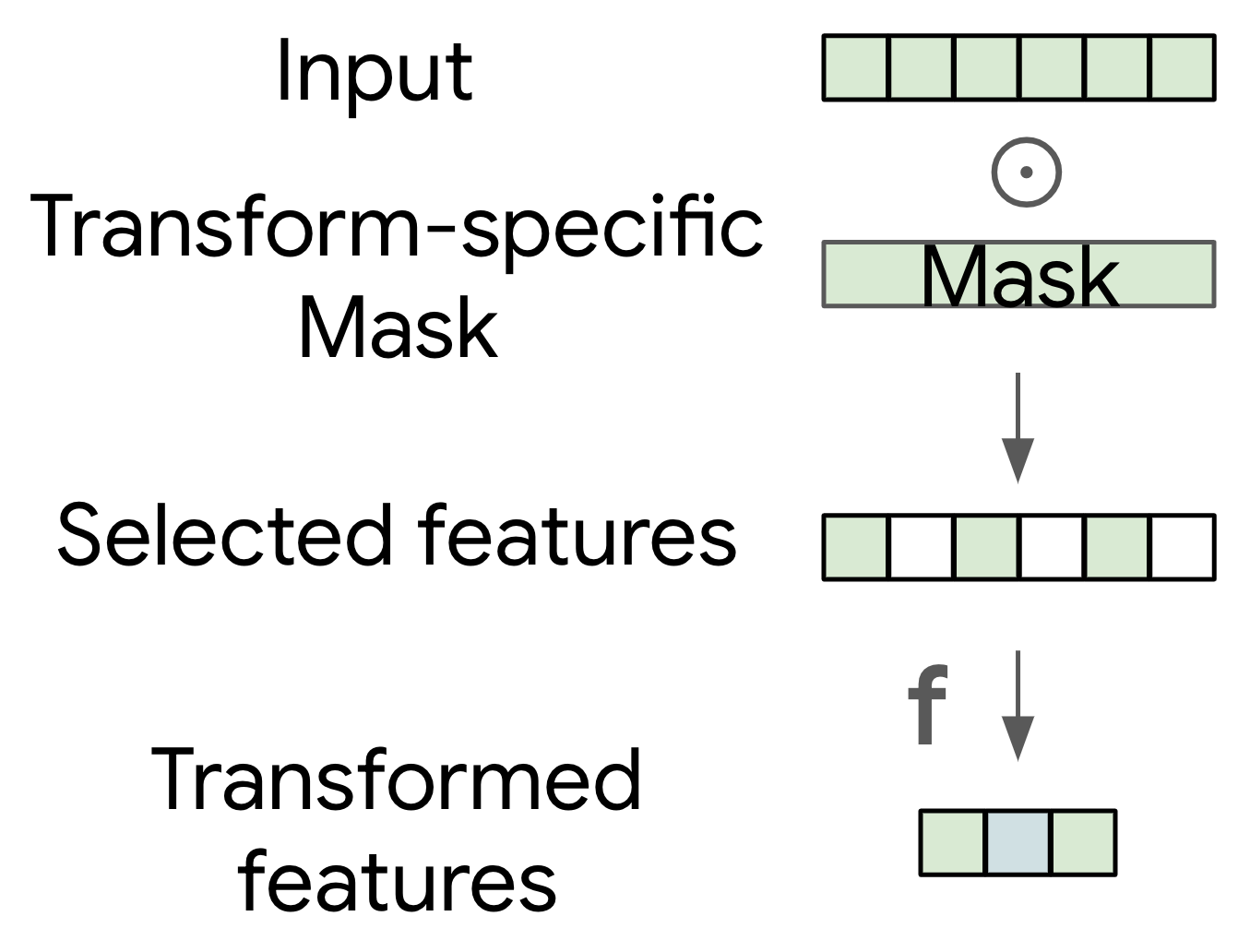}
\caption{Learnable selection of features for a given transform $f$ via local feature importance masking, where $\odot$ denotes element-wise multiplication. Here, three features are selected and weighted by the mask and input into $f$.}
\label{fig:local_mask}
\end{figure}

\subsection{Learning feature importance masks}
Regarding making the actual search for new features implicit, i.e. without explicitly manifesting all possible transform-feature combinations, we propose a learnable feature importance masking mechanism.

Specifically, \ours{} learns two types of feature importance masks: 1) local masks $\{\bbm_{loc}^{(f)}\}_{f\in \bbF}$ on input features for each candidate transform function $f$, illustrated in Figure~\ref{fig:local_mask}, and 2) a global mask $\bbm_{glb}$ on the concatenated output of the all transform functions, illustrated in Figure~\ref{fig:global_mask}.

Given input $\bbx \in \bbX$, we learn a set of local feature importance masks $\{\bbm_{loc}^{(f)}\}_{f\in \bbF}$, one per transform function. We subsequently obtain a weighted mask $\bbm^{(f)}$ for each $\bbm_{loc}^{(f)}$, by first applying softmax, and then discarding all but the most important, i.e. highest-valued, entries to $\bbm_{loc}^{(f)}$. Next, the weighted masks $\bbm^{(f)}$ are applied to $\bbx$ elementwise: $\bbx^{(f)}_{loc} =\bbm^{(f)} \odot \bbx$. %$\bbx_{sp}^{(f)}$ is therefore also sparse,
Since $\bbm_{feat}$ are sparse, $\bbm_{feat} \odot \bbx$ are also sparse,
and would contain only features that are learned to be most applicable for a given transform. $\bbm^{(f)}$ are continuous probability vectors -- thus, the space of feature-transform pairs is continuously explored, without explicitly realizing each feature-transform combination.

In summary, the feature importance masks based transformation with input $\bbX$, learned (unnormalized) mask $\bbm$, and top $h$ features selected for the transform $f$ consist of:
\begin{align*}
    I_h &:= \textrm{top\_h}(\bbm)  &\triangleleft \text{Select indices of top $h$ elements.} \\
    \bbm_{I_h} &= \textrm{Select}(\bbm, I_h) &\triangleleft \text{Select top $h$ elements.}\\
    \bbm &= \textrm{Softmax}(\bbm_{I_h})    &\triangleleft \text{Normalize those top $h$ entries. } \\
    \bbX_{I_h} &= \textrm{Select}(\bbX, I_h) &\triangleleft \text{Select top $h$ features.}\\
    \bbX &= \bbX_{I_h} \odot \bbm    &\triangleleft \text{Weigh top features.}\\
    \bbX &= f(\bbX)     &\triangleleft \text{Apply transform.}
\end{align*}

Subsequently, all transformed features $\{\bbx_{loc}^{(f)}\}_{f\in \bbF} $ are concatenated, and a global feature importance mask $\bbm_{glb}$ is learned and applied to the global concatenated set of features, exactly the same way as the local masks $\bbm_{loc}$, to select the most relevant feature-transform combinations.  

In other words, the local masks  $\{\bbm_{loc}^{(f)}\}_{f\in \bbF}$ select which features to input for each transform, and the global mask $\bbm_{glb} $ is learned across all feature-transforms. 

\begin{figure}%[h]
\centering
\includegraphics[width=0.4\linewidth]{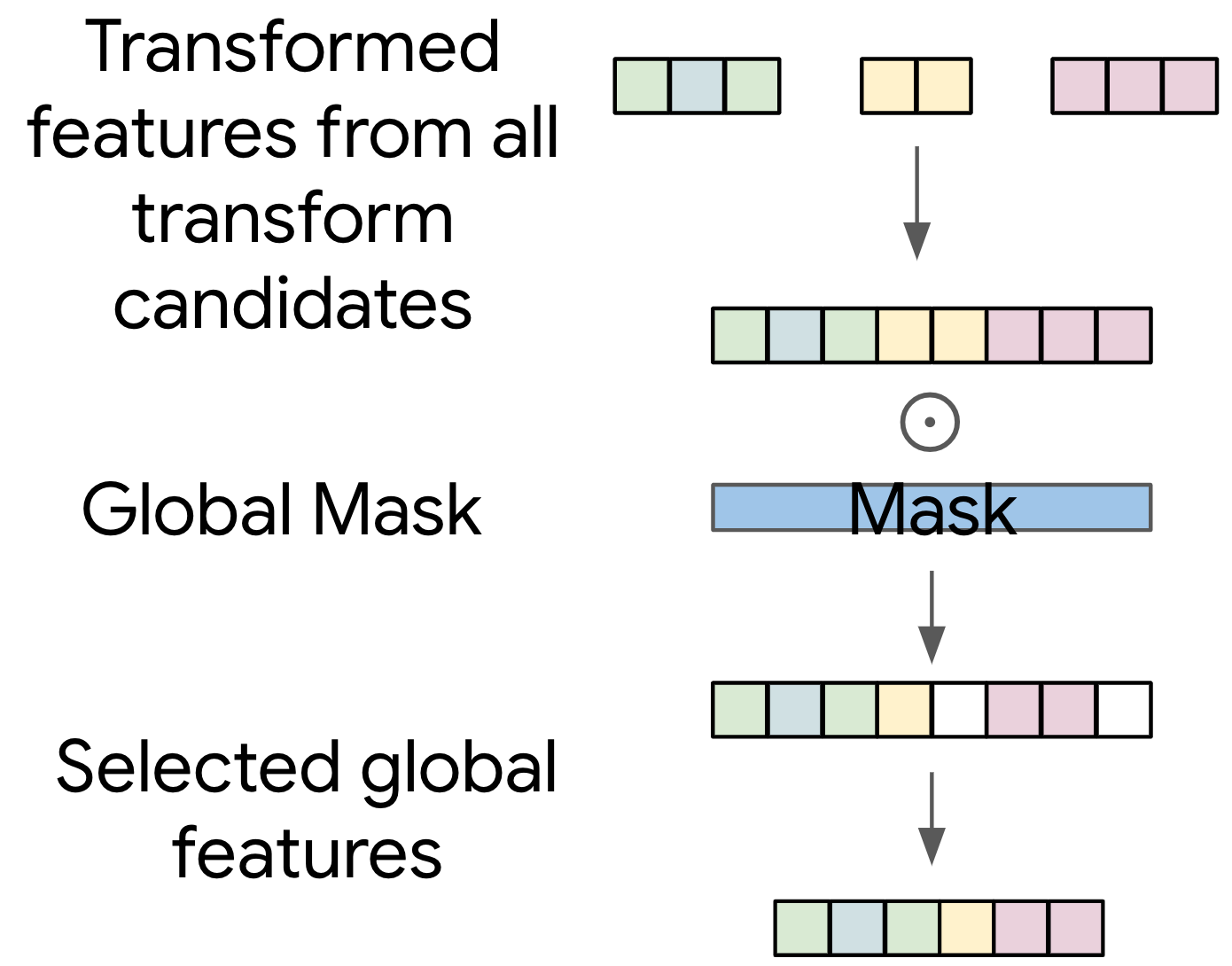}
\caption{Using a global mask to select the global features from the set of transformed feature candidates, optimized for the downstream task performance. Note that since the number of top selected elements for each local and global mask is fixed a priori, the input to the predictor does not vary in size during training.}
\label{fig:global_mask}
\end{figure}

\begin{algorithm}[tb]
   \caption{Feature engineering using \ours{}.}
   \label{alg:afe}
\begin{algorithmic}
   \STATE {\bfseries Input:} input data $\bbX$, labels $\bm{Y}$, $k$ transform function candidates $\{f_i\}$, number of training steps $T$, top $p$ selected features per feature mask.
   \STATE {\bfseries Output:} Engineered features.
   \STATE {\bfseries Initialize:} Learnable local and global masks $\{\bbm_{loc}\}_{i=1}^k$ and $\bbm_{glb}$; initialize uniformly randomly.
   \FOR{step 1 {\bfseries to} $T$} 
   \FOR{$f$ {\bfseries in} $\{f_i\} $} 
   \STATE Generate {\bfseries mask} $\bbm^{(f)}$ by applying softmax on $\bbm_{loc}^{(f)}$, and discarding non-top-$p$ entries in the output.
   \STATE {\bfseries Select} and {\bfseries weight} input features using mask: $\mathbf{X}_{loc}^{(f)} = \mathbf{X} \odot \bbm^{(f)}$.
   \STATE Apply transform: $\bbX^{(f)} \leftarrow f\left(\mathbf{X}_{loc}^{(f)}\right)$. 
   \ENDFOR
   \STATE Concatenate $\{\mathbf{X}^{(f)}\}$ to generate $\bbX_{glb}$.
   \STATE Append $\bbX_{tmp}$ output from Algorithm~\ref{alg:timeseries} to $\bbX_{glb}$ for time series features.
   \STATE $\hat{\bbX} \leftarrow \bbm_{glb} \odot \bbX_{glb}$ . 
   \STATE Input $\hat{\bbX}$ into MLP task predictor and compute loss against ground truth $\bm{Y}$.
\ENDFOR
\STATE Output transformed features $\hat{\bbX}$.
\end{algorithmic}
\end{algorithm}

\begin{algorithm}[tb]
   \caption{Extension of feature engineering with \ours{} for time series.}
   \label{alg:timeseries}
\begin{algorithmic}
   \STATE {\bfseries Input:} input time series data $\bbX$, $k$ temporal transform function candidates $\{f_t\}_{t=1}^k$, top $p$ selected features per feature mask. 
   \STATE {\bfseries Output:} Engineered time series features $\bbX_{tmp}$ for the current iteration.
   \STATE {\bfseries Initialize:} Learnable temporal masks $\{\bbm_{t}\}_{t=1}^k$ uniformly randomly.
   %\FOR{step 1 {\bfseries to} $T$} 
   \FOR{$f$ {\bfseries in} $\{f_t\}_{t=1}^k$} 
   \STATE Generate {\bfseries mask} $\bbm^{(f)}$ by applying softmax on $\bbm_{t}^{(f)}$, and discarding non-top-$p$ entries in the output.
   \STATE {\bfseries Select} and {\bfseries weight} input features with mask: $\mathbf{X}_{t}^{(f)} = \mathbf{X} \odot \bbm^{(f)}$.
   \STATE Apply temporal transform: $\bbX_t \leftarrow$  $f(\mathbf{X}_{t}^{(f)})$. 
   \ENDFOR
   \STATE $\bbX_{tmp} \leftarrow$ Concatenate $\{\mathbf{X}_{t}\}_{t=1}^k$. 
   \STATE Output transformed time series features $\bbX_{tmp}$. 
%\ENDFOR
\end{algorithmic}
\end{algorithm}

The use of masking on the transform inputs turns discrete search space for feature transforms to a continuous one. Specifically, the model learns the probability of applying a given transform to an input feature. This allows a continuous exploration of which subset of features to apply a transform to, \textit{without the need} to explicitly realize every discrete subset of features a transform should be applied to.
In addition, such a continuous exploration also eases optimization and training convergence.

The probability masks also act as a normalization to the input features, which can mitigate distribution shift across batches and faciliate training convergence \citep{ba2016layer}.

Finally, the local and global masks are learned with the cross entropy loss for classification tasks, and the mean absolute error loss for regression tasks.

\subsection{Extending feature discovery to time series}

We extend the capabilities of \ours{} to time series data, by proposing learning temporal feature engineering via temporal masks.
%\textbf{Time series transforms.}
\ours{} can be readily adapted to time series data, by learning temporal masks along the time dimension, selecting the most relevant time steps for a given transform. For instance, for the operator that selects data from a particular past time step, such as when the sales data one week before the current time point is most indicative of current sales, a temporal mask over historical values can learn to select the time steps most relevant to the task target. Additional learnable transforms can include aggregations and differencing along the temporal dimension. \S\ref{sec:transforms} details all time series transformations. Algorithm~\ref{alg:timeseries} overviews the discovery process for time series.
The ease of learning such temporal masks indicates that we can readily integrate learning time series transforms within the \textit{same} end-to-end framework as for other data types, which is useful as many real world datasets might contain multiple modalities per sample.

\subsection{Transform functions explored}
\label{sec:transforms}
\ours{} explores a limited, but diverse, array of transformations. This list includes each function's learned components. The target number of selected features by the local mask to apply each transform on depends on the compute resources available. In our experiments, we set the number of selected features to be 5, as a general-purpose quantity that works well empirically, and to reduce hyperparameter tuning. The same set of transform candidates is used for all tasks.

\begin{itemize}
    \item \textit{Polynomial transforms}: learnable coefficient and degree for each monomial. 
    \item \textit{Logarithm}: natural logarithm.
    \item \textit{Custom z-scaling}: the multiplier for division and mean for subtraction are both learned.
    \item \textit{Additive aggregation}: the features that have been selected by the learnable local mask are added for each sample.
    \item \textit{Multiplicative aggregation}: the features that have been selected by the learnable local mask are multiplied for each sample.
    \item \textit{Gaussian}: learnable mean and standard deviation parameters.
    \item \textit{Quantile transform}: the data are bucketed into four quantiles.
    \item \textit{GroupBy}: aggregation based on categorical features. The categories to aggregate based on is learned.
    \item \textit{Aggregation along temporal dimension}: for time series.
    \item \textit{Temporal standard normalization}: for time series.
    \item \textit{Temporal differencing}: difference between adjacent steps.
    \item \textit{Temporal lag}: learnable lag along the temporal dimension.
    \item \textit{Relative temporal mean}: mean of the normalized time series.
    \item \textit{Temporal difference}: difference by specified number of steps.
    \item \textit{Temporal mean}: mean of the (unnormalized) time series.
    \item \textit{Identity}: this identity transform allows important raw features to be used for the downstream task.
\end{itemize}

The selection of this human-expert-curated set of transform functions is  motivated by their applicability in real world use cases. We note that \ours{} is agnostic to this function space, and additional functions can be added. 

\subsection{Complexity analysis}
\label{sec:complexity}
Let $N$ be the number of training samples, $m$ the initial number of features, and $k$ the number of functions in the initial set of transforms. 
\ours's overall complexity is $O(Nmk)$, as one mask is learned per transform function for all features.
Notably, \ours{} scales linearly with respect to the number of initial features $m$, as opposed to quadratic in $m$ as in some prior works \citep{openfe}. In addition, feature importance masking allows \ours{} to scale linearly with the number of transform functions $k$, demonstrating effective search in the large candidate space.

Table~\ref{tab:timing} shows runtime comparison on datasets of varying numbers of features. As shown, while the timing results are similar across methods on the Mice dataset with 77 features, the timing difference drastically widens on the Isolet dataset with 617 features. This is consistent with the timing complexity in theory. 
Settings such as the number of paralell processes were experimented with on baseline methods to obtain the best latency results.

\section{Motivation for limiting transform functions}
\label{sec:expressiveness}

A guiding principle in \ours{}'s design is exploration efficiency for the transformation space, which consists of a pre-determined set of human expert curated transform functions. 
The overall motivation is that a representative subset of transform functions suffices to learn expressive feature transformations for many real world tasks. Intuitively, learning a set of ``basis functions``, such as Gaussians, in the transform space suffices to span the relevant transform space. We make this intuition rigorous in Lemma~\ref{lem:basis-function}.

\begin{lemma}
\label{lem:basis-function}
Given a continuous function $g: [-N, N]^n\to \Rbb$, for any $N > 0$, and any $\epsilon > 0$, there exists a finite sequence of reals $c_i$, and a finite sequence of Gaussian functions $f_i$, such that $|g(x) - \sum_i c_i f_i (x)| < \epsilon$ for all $x\in [-N, N]^n$.
\end{lemma}

\begin{proof}
We apply the generalized Stone-Weierstrass Theorem \citep{Stone_Ross_2007} to locally compact spaces (stated in Appendix \S\ref{sec:stone_weierstrass}). We first check that all its conditions are satisfied.

Let $C_0(\Rbb^n, \Rbb)$ be the space of real-valued continuous functions $\Rbb^n\to \Rbb$ that vanish at infinity. The set of continuous functions $g: [-N, N]^n\to \Rbb$ coincides with elements of $C_0(\Rbb^n, \Rbb)$ restricted to the domain $[-N, N]^n$. 
%This is because we can continuously extend any such function $g$ to a function on all of $\Rbb^n$ that vanishes at infinity, by interpolating between $g$ at the boundary of $[-N, N]^n$ and $0$.

By Lemma~\ref{claim:algebra}, the set of Gaussian functions on $[-N, N]^n$ form an algebra $A$ that vanishes nowhere and separates points. 
Therefore, since the space of continuous functions $g: [-N, N]^n\to \Rbb$ coincides with $C_0(\Rbb^n, \Rbb)$ on $[-N, N]^n$, and the algebea $A$ is dense in $C_0(\Rbb^n, \Rbb)$ by the locally compact version of the Stone-Weierstrass Theorem. Furthermore, $A$ is dense in the space of continuous functions $g: [-N, N]^n\to \Rbb$ in the topology of \textit{uniform convergence}. In other words, given any $\epsilon > 0$, there exists a finite sum of Gaussians $\sum_i c_i f_i (x)$ such that $|g(x) - \sum_i c_i f_i (x)| < \epsilon$ for all $x\in [-N, N]^n$.
\end{proof}

The following lemma, proved in Appendix~\S\ref{sec:appendix}, is used in the proof above:
\begin{lemma}
\label{claim:algebra}
The set of Gaussian functions on $\Rbb^n$ form an associative algebra $A$. $A$ vanishes nowhere and separates points.
\end{lemma}

%Lemma \ref{lem:basis-function} states that a function on the reals can be approximated by radial basis functions arbitrarily closely.

Lemma~\ref{lem:basis-function} states that, in theory, just \textit{one} particular type of functions, namely Gaussian functions, suffices to approximate any continuous target transforms well, where target transforms refer to any set of transform functions that achieve optimal downstream task performance given the model capacity.  Note that while Lemma~\ref{lem:basis-function} applies to Gaussian approximations of continuous transforms, and in many real world applications discrete functions are often used for feature transformation. This nonetheless gives reason to believe that exhaustively searching a large function space is not necessary for finding effective feature transforms.

%restricting the transform function space to a small family of functions, 

In practice, \ours{} is applied to a small family of learnable functions besides Gaussians, detailed in \S\ref{sec:transforms}, to reduce the total number of functions needed to approximate the target transforms well, and to include discrete transformations not easily approximated by a few continuous functions. For example, we consider polynomial transforms, where the degrees and coefficients of the polynomial can be learned end-to-end, as well as aggregations grouped by categorical features. Experiments show that in practice, a small number of learnable functions from few different families, work well to achieve the state-of-the-art results.
Note that this motivation leads us to a markedly different approach from prior works such as \cite{openfe} or \cite{fetch}, which explore an exhaustively large space of candidate functions and their compositions, and focuses on exploring this space efficiently.

%While the number of basis functions required can be large, and the number of radial basis functions learned in practice is finite, Claim~\ref{lem:basis-function} states that the more radial basis functions are learned in the network, the more precise the approximation will be. 

\section{Experiments and Results}

\subsection{Experiments} We experiment the effectiveness of \ours{} for a variety of prediction tasks, compared to recent competitive baselines OpenFE \citep{openfe} and AutoFeat \citep{autofeat}. For fair comparison, we apply the same pipeline across all methods -- each method learns a set of feature transform functions on the training set. These learned transform functions are then applied to both the train and test sets. 

We evaluate the performance of the predictor models on the transformed dataset. %We focus on two types of predictors, MLP and XGBoost.
For each experiment, we fix the same predictor for all baseline methods. The goal is to use a simple, commonly-used predictor to test the effectiveness of the generated features. The same predictor is used across all baselines. The predictor is deliberately kept simple, to make the feature generator as the main performance driver, rather than the predictor.

We experiment with two commonly-used types of predictors: a two-layer MLP and XGBoost, to test the efficacy of the generated features independent of the downstream predictor. The reported performance is averaged across ten trials, each with different random seeds for MLP, and different randomly generated seeds, tree depths, and ensemble sizes for XGBoost.
%This is repeated ten times with different randomly selected configurations. 
This is done to avoid the scenario where a single configuration might disproportionately benefit a particular feature engineering method. We do not restrict the number of features generated for any baseline method, to leverage each method to the fullest extent.

In addition, we compare the performance on the discovered features against the performance on the original dataset. We perform this comparison on four datasets of varying numbers of samples and features: Mice, Isolet, MNIST, and Fraud. Furthermore, we compare the latency performance across all methods on two datasets of varying scales.

\subsection{Datasets} We experiment with a variety of datasets, with varying number and type of features, to test the efficacy of the feature engineering methods. This includes Isolet, a speech dataset; Diabetes and Heart Disease prediction datasets, both datasets predicting medical conditions; Fraud, a finance dataset on sales transactions; MNIST, an image dataset; Coil, an image dataset; Mice, a protein expression dataset; and M5, a time series dataset in retail.
Table~\ref{tab:data-char} 
%in Appendix~\ref{sec:appendix} 
overviews the dataset statistics.

%\subsubsection{Dataset Details}
%This subsection contains further details on the experimental datasets.
The Mice dataset consists of protein expression levels measured in the cortex of normal and trisomic mice who had been exposed to different experimental conditions. Each feature is the expression level of one protein.
Isolet consists of preprocessed speech data of people speaking the names of the letters in the English alphabet with each feature being one of the preprocessed quantities, including spectral coefficients and sonorant features.
Coil consists of centered gray-scale images of 20 objects taken at certain pose intervals, hence the features are image pixels.
Activity consists of sensor data collected from
a smartphone mounted on subjects while they performed several activities such as walking or standing. MNIST consists of 28-by-28 grayscale images of hand-written digits and clothing items. Each pixel is treated as a separate feature.
M5 \citep{m5-forecasting-accuracy} consists of sales data from the retailer Walmart for the task of forecasting daily sales.
%Diabetes and Heart Disease prediction datasets.
In addition, we consider the Ames housing dataset \citep{amesDataset}, with the goal of predicting residential housing prices based on each home's features; as well as the IEEE-CIS Fraud Detection dataset \citep{kaggle_fraud}, with the goal of identifying fraudulent transactions from numerous transaction and identity dependent features.

\begin{table}%[!htbp]
\caption{Attributes of datasets used in experiments. 
}
\centering
\begin{tabular}{lccc}
\toprule
\textbf{Dataset}       & \textbf{\# samples} & \textbf{\# features} & \textbf{\# classes} \\ \midrule %\hline
Mice          & 1080              & 77                 & 8                        \\ %\hline
Isolet        & 7797              & 617                & 26                       \\ %\hline
Coil-20       & 1440              & 400                & 20                       \\ %\hline
%Activity      & 5744              & 561                & 6                        \\ \hline
Ames      & 1460              & 81                & N/A                        \\ %\hline
Diabetes      & 768              & 9                & 2                        \\ %\hline
Heart Disease      & 270              & 14                & 2                       \\ %\hline
Fraud      & 100000              & 681                & 2                        \\ %\hline
MNIST         & 10000             & 784                & 10                       \\ %\hline
%Synthetic     & XXX              & XXX                & 2                        \\ \hline
\bottomrule
\end{tabular}
\label{tab:data-char}
\end{table}

\subsection{Results}
On datasets across a variety of domains, \ours{} consistently outperforms other baselines. Table~\ref{tab:baseline-mlp} shows the prediction results using engineered features with the MLP model as the predictor, and Table~\ref{tab:baseline-xgb} shows the results with the XGBoost model as the predictor. 

\ours's superior performance is particularly prominent for the MLP predictor. This is aligned with the training of the feature engineering method, underlining the advantage of end-to-end learning. 

\begin{table}[t]
\caption{Comparison against baselines on various datasets with the \textit{MLP predictor}. Each feature engineering method is trained to discover dataset-specific new features, then a deliberately simple MLP predictor is trained on top of the transformed dataset with the new features. This is repeated ten times with different randomly selected MLP configurations, to avoid the scenario where a single configuration might disproportionately benefit a particular feature engineering method. We do not restrict the number of features generated for any baseline method. As shown, \ours{} outperforms competitive baselines, demonstrating its advantages in end-to-end feature engineering.}
\label{tab:baseline-mlp}
\vskip 0.15in
\begin{center}
\begin{small}
\begin{sc}
\resizebox{.75\linewidth}{!}{
\begin{tabular}{lccccr}
\toprule
Dataset & \ours{} (ours) & OpenFE & AutoFeat  \\
\midrule
Diabetes $\uparrow$ & \textbf{88.4}$\pm$ 1.8& 82.3$\pm$ 1.3 & 72.55$\pm$ 1.7\\
Isolet $\uparrow$ & \textbf{94.8}$\pm$ 0.9& \textbf{94.4}$\pm$ 1.2 & 93.1$\pm$ 2.1 \\
Mice $\uparrow$& \textbf{97.7}$\pm$ 1.1& 93.1$\pm$ 1.6 & 87.04 $\pm$ 1.8 \\
Heart disease $\uparrow$& \textbf{92.5}$\pm$ 1.7&  84.4$\pm$ 1.6 & 78.2 $\pm$ 5.3\\
Coil $\uparrow$& \textbf{98.9} $\pm$ 1.7 & 85.8$\pm$ 1.2 & \textbf{99.3}$\pm$1.4 \\
Ames $\downarrow$& \textbf{0.74} $\pm$ 0.08 &  1.07 $\pm$ 0.09 & 1.268 $\pm$ 0.15 \\
\bottomrule
\end{tabular}
}
\end{sc}
\end{small}
\end{center}
%\vskip -0.1in
\end{table}

\begin{table}[t]
\caption{Baseline comparisons with the \textit{XGBoost predictor}. 
This is done in addition to the MLP predictor, to test the efficacy of the generated features independent of the downstream predictor.
%The feature transformations discovered by \ours{} has superior or competitive performance over state-of-the-art baselines, regardless of the downstream predictor. 
}
\label{tab:baseline-xgb}
\vskip 0.15in
\begin{center}
\begin{small}
\begin{sc}
\resizebox{.75\linewidth}{!}{
\begin{tabular}{lccccr}
\toprule
Dataset & \ours{} (ours) & OpenFE & AutoFeat  \\
\midrule
Diabetes $\uparrow$ & \textbf{77.7}$\pm$ 2.3& 74.3$\pm$ 4.3 & 72.8 $\pm$ 2.2\\
Isolet $\uparrow$ & \textbf{91.8}$\pm$ 0.6& 90.5$\pm$ 1.2 & 90.6 $\pm$ 1.1 \\
Mice $\uparrow$ & \textbf{95.7}$\pm$ 2.1& 94.6 $\pm$1.8 & 90.8 $\pm$ 3.8 \\
Heart disease $\uparrow$ & 80.8$\pm$ 5.4& \textbf{81.4}$\pm$ 5.3 & 79.3 $\pm$ 6.1 \\
Coil $\uparrow$ & 96.1 $\pm$ 1.5 & \textbf{97.1} $\pm$ 1.3 & 96.7 $\pm$ 1.1  \\
Ames $\downarrow$ & \textbf{0.22} $\pm$ 0.03 &  0.31 $\pm$ 0.09 & 0.30 $\pm$ 0.09 \\
\bottomrule
\end{tabular}
}
\end{sc}
\end{small}
\end{center}
%\vskip -0.1in
\end{table}

Table~\ref{tab:timing} shows the latency on two datasets with varying numbers of features. As described in Sec.~\ref{sec:complexity}, the masking mechanism means \ours{}  scales linearly with respect to the number of initial features, and a pre-determined human-expert-curated set of transform functions greatly improves the efficiency of the transform search space. As a result, \ours{} demonstrates superior or on-par performance compared to baselines despite notably lower latency.

Table~\ref{tab:m5} shows the regression mean absolute error on the time series dataset M5, benchmarked with the MLP predictor. \ours{} is able to preserve the temporal relations in the time series features, by learning feature importance masks along the temporal dimension.

\begin{table}%[t]
\caption{Runtime values (in GPU-minutes) on datasets with varying numbers of features (indicated in parentheses). Benchmarking is done on one V100 GPU with a 2.20GHz 32-core processor. \ours{} scales linearly with respect to the number of input features given the linear scalability of learning importance masks. The pre-determined human-expert-curated set of transform functions greatly improves the search in the transform space. While the runtime on Mice are comparable, the time saving significantly grows for Isolet, a dataset with an order of magnitude larger number of features. This is consistent with the theoretical complexity discussed. A $\times$ indicates timeout, exceeding the two-day limit.}
\label{tab:timing}
\vskip 0.15in
\begin{center}
\begin{small}
\begin{sc}
\resizebox{.65\linewidth}{!}{
\begin{tabular}{lccccr}
\toprule
Dataset & \ours{} (ours) & OpenFE & AutoFeat  \\
\midrule
%Mice & \textbf{22.7}& 139.6 & 282.29  \\
Mice (77) & \textbf{0.38}& 2.3 & 4.7  \\
Isolet (617) & \textbf{4.60}& 2392 & $\times$  \\
%Heart disease $\uparrow$& \textbf{92.5}$\pm$ 1.7&  84.4$\pm$ 1.6 & 78.2 $\pm$ 5.3\\
%Coil $\uparrow$& \textbf{98.9} $\pm$ 1.7 & 85.8$\pm$ 1.2 & \textbf{99.3}$\pm$1.4 \\
\bottomrule
\end{tabular}
}
\end{sc}
\end{small}
\end{center}
\vskip -0.1in
\end{table}
%(6238, 19253) on mice. originally 77 features. automan feat

\begin{table}[t]
\caption{MAE results on the time series dataset M5, using an MLP predictor. \ours{} is able to natively ingest time series data, discoverying meaningful feature transformations along the temporal dimension, such as time lag or timewise aggregation. This helps preserve temporal relations within the data during feature discovery.}
\label{tab:m5}
\vskip 0.15in
\begin{center}
\begin{small}
\begin{sc}
\resizebox{.7\linewidth}{!}{
\begin{tabular}{lccccr}
\toprule
Dataset & \ours{} (ours) & OpenFE & AutoFeat  \\
\midrule
%Mice & \textbf{22.7}& 139.6 & 180.1  \\
M5 $\downarrow$ & \textbf{0.870}$\pm$ 0.05 & 0.966 $\pm$ 0.07 & 1.146$\pm$ 0.13  \\
\bottomrule
\end{tabular}
}
\end{sc}
\end{small}
\end{center}
\vskip -0.1in
\end{table}
%\yd{benchmark m5 with xgboost}

\begin{figure}
\centering
\includegraphics[width=.9\linewidth]{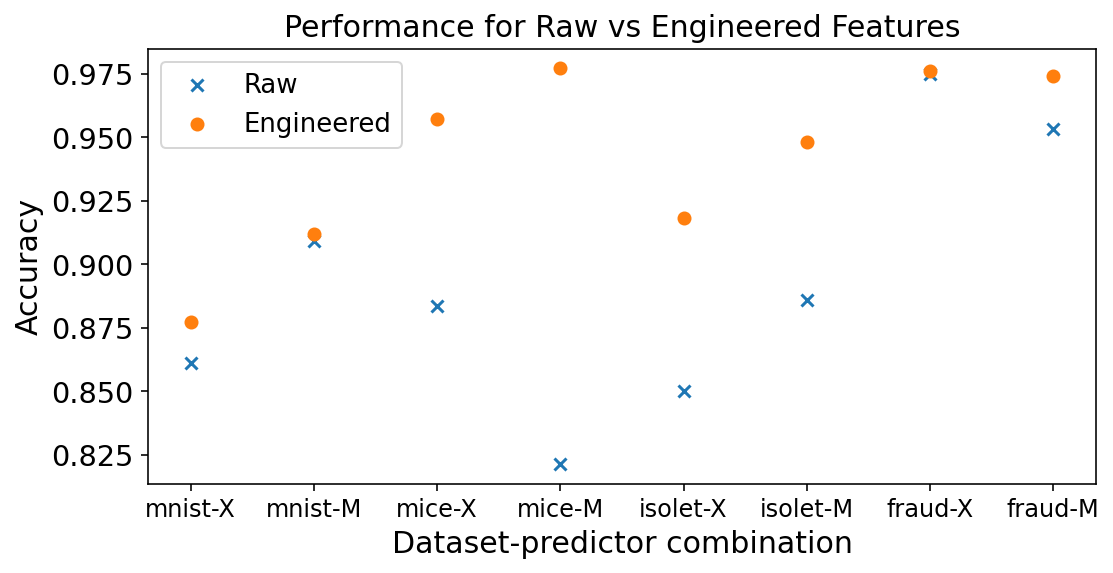}
\caption{Performance comparisons between raw and \ours{} engineered features. Higher is better. MLP (indicated by "M" postfix) and XGBoost  (indicated by "X" postfix) predictors are trained on either the raw or engineered features. Interestingly, the performance gain is larger for smaller datasets with fewer samples and features (Mice and Isolet) than for larger datasets (MNIST and Fraud).} \label{fig:raw_eng_feat_comp}
\end{figure}

\paragraph{Performance improvement with engineered features.} Figure~\ref{fig:raw_eng_feat_comp} shows the average performance improvement given the \ours{} engineered features over raw features.
The task performance with the engineered features either are on par or 
for both predictors.
Perhaps not surprisingly, the improvement is larger for the MLP predictor than for the XGBoost predictor, as \ours{} is trained end-to-end with an MLP predictor, hence the engineered feature space is customized for an MLP. Interestingly, the gain is consistently larger for the smaller datasets than for the larger datasets. For instance, the performance gain with engineered features on the Mice dataset, with 1080 samples and 77 features, is 7.5X the performance gain than on the MNIST, with 10,000 samples and 784 features. One hypothesis for this difference is that the feature space for the larger datasets with more samples and features is already saturated in the original input, whereas data augmentation complements the original feature space for smaller datasets with fewer samples and features. We leave this interesting direction to future work. %\yd{cite work on data augmentaiton?}

\paragraph{Qualitative analysis on engineered features.} 
While there is not a single optimal set of engineered features with respect to task performance (for instance different numbers of additional features allowed would result in different optimal sets of engineered features), some \ours{} engineered features would exhibit common sense. For instance, on the M5 dataset, for predicting the future sales volume, the engineered features \textit{TemporalMean(sales)}, referring to the average past sales quantity, \textit{TemporalLag(sales, 2)}, referring to the sales quantity two timesteps ago, as well as \textit{QuantileTransform(sell\_price)}, referring to a quantile transform of the sell price, were among the top generated features. 
Similarly, for the Ames housing price dataset, quantile transforms of square footage features are among top learned transforms.
We note that as \ours{} optimizes for the downstream objective directly, not all generated features are human interpretable, e.g. the generated feature \textit{PolynomialTransform(pixel\_73)} does not lend itself to much interpretation.

\section{Related work}

Many existing works on automated feature engineering employ variations of the expansion-reduction framework, which first produces a large candidate features set, and subsequently selects the most salient features from the candidate pool based on their effects on model performance \citep{kaul_afe, lam_afe}. 
OpenFE first generates candidate features by applying hand-selected transform operators on the raw feature set. It then prunes less effective features via feature boosting and a coarse-to-fine two-stage pruning mechanism. Its feature boosting algorithm incrementally evaluates a feature's effectiveness without retraining the model, though OpenFE still materializes a large candidate features pool. %For instance, on a dataset with 6237 samples and 617 features, OpenFE generated over 1.3 million candidate features. 
%As such, OpenFE has a complexity that is quadratic in terms of the feature size.
FETCH \citep{fetch} uses reinforcement learning to learn a policy neural network that outputs feature engineering actions pertaining to a given dataset, and seeks to achieve transferability across datasets. \citep{wang_reinforce_afe} also employs reinforcement learning towards automated feature engineering, by developing a feature pre-evaluation model to improve sample and feature efficiency and warm-start policy training.
AutoFeat \cite{autofeat} aims to perform feature engineering and selection while retaining interpretability of linear models, it does so by generating a large pool of non-linear features first, and then selecting a subset of the most useful features. Like \ours, AutoFeat also leverages the task objective directly during feature engineering. Neural Feature Search \citep{nas} adopts NAS techniques for automated feature engineering.  
On time series data, \citep{timeseries_afe} models multivariate time series as a cumulative sum of trajectory increments to engineer time series features at scale. 
In contrast to many such works, \ours{} does not explicitly manifest the feature transforms explored, but rather explores a continuous space that contains soft combinations of feature transforms. Thus, it yields effective discovery of features even from large candidate function spaces efficiently.

\section{Discussion}

\paragraph{Feature importance masking.}
\ours{} relies crucially on learned feature importance masks -- local masks to select and weigh the features for each transform function, and global masks to select and weigh across all transformed features. Such mask-based weighting strategy has been employed to great success in deep learning, most notably in the attention-based architectures like Transformers \citep{vaswani}, where attention scores are learned to weigh key, query, and value features generated from the input. \citep{slm} learns sparse masks to select the most important features for a downstream task. Indeed, importance mask learning exhibits numerous advantages, such as ease of optimization, allowing continuous search in feature space even for discrete outcomes, and enabling end-to-end backpropagation using downstream task objectives. To our knowledge, \ours{} is the first to apply mask-based importance weighting to feature engineering and discovery. 

\paragraph{Transform function search space.} 
Much existing work follow the expand-and-select approach, where an expansive set of transforms are incrementally applied to all possible features, and the resulting large pool of transformed features is whittled down based on model performance. Existing works have optimized this process, such as developing incremental training algorithms that do not require re-training between successive feature relevance evaluations \citep{openfe}. %\yd{more citations}
However, many such optimizations still require manifesting an expansive set of transformed features.
For instance, on the Isolet data with 617 features, \citep{openfe} creates over 1.3 million candidate features, whereas \ours{} has a non-explicit search space that has latent feature dimension 256 at the widest. Here, non-explicit refers to the fact that the full pool of feature transforms is never explicitly manifested, but rather explored continuously via feature importance masks. 

%(6238, 19253) on mice. originally 77 features. automan feats

\paragraph{Interpretability of temporal features.}
\ours{} is able to natively ingest time series data, preserving the temporal relations within each series, rather than treating each time step independently. Specifically, \ours{} can discover meaningful feature transformations along the time dimension, such as those based on time lags or timewise aggregations. %This contrasts with existing approaches that treat each time step as an individual feature.
For instance, on the M5 dataset, one discovered feature is \textit{TemporalLag(sales, 2)}, indicating that the feature sales from two time steps before the current one is a relevant feature for the downstream task. The structure of such transforms enabled by the masking design add a layer of interpretability to the discovered features.

%'TemporalLag(t1, 2)', 'poly_transform(x1)', 'poly_transform(x2)', 'RelTemporalMean(series0)', 'quantilize(x3)', 'TemporalLag(t1, 3)', 'TemporalDiff(t0, t1, t2, t3, t4; 1)', 'TemporalMean(series1)'

\paragraph{Limitations and future work.} While \ours{} has a myriad of advantages with regard to performance and latency, the pursuit for low latency entails certain sacrifices for flexibility. Specifically, since the local masks for different transforms are learned independently, transforms cannot be composed unless specific compositions are included within the initial set of transforms. Nonetheless, experimental results indicate that this constraint does not meaningfully decrease model performance. Furthermore, Lemma~\ref{lem:basis-function} suggests that functions from a single family suffice to approximate arbitrary continuous transforms. 
%in the network forward pass sequentially, this fixes the ordering of functions in composite transformations. E.g. whether the composite ordering is $\log(\sqrt{x})$ or $\sqrt{\log(x)}$, depends on whether the mask for $\log(x)$ or the mask for $\sqrt{x}$ is learned first.

Currently, the initial set of transforms to be explored include many commonly-used transform functions, to further optimize this set, it would be useful to create an automated way to determine what to include in this set.
Furthermore, \ours{} is shown to be able to discover features for datasets that include time series.
Extensions to other modalities in the same way is an important direction,  e.g. for images, a pixel-wise feature importance mask can be learned for selecting the most pertinent features for a given transform function.
%An important future direction is to incorporate additional modalities, such as vision, into this framework and demonstrate model performance improvements via feature engineering.

\bibliography{references}
\bibliographystyle{ACM-Reference-Format}

\newpage
\appendix
\onecolumn
\input{appendix}
\label{sec:appendix}

\end{document}

%% file: math_commands.tex
\theoremstyle{remark}

\theoremstyle{definition}

\usepackage{bm}

\newcommand{\bbm}{\bm{m}}

\newcommand{\Rbb}{\mathbb{R}}

\newcommand{\rone}[1]{}
\newcommand{\rtwo}[1]{}
\newcommand{\rthree}[1]{}

%% file: appendix.tex
\section{Appendix}

\subsection{The Stone-Weierstrass Theorem}
\label{sec:stone_weierstrass}

In \S\ref{sec:expressiveness} we motivate the use of a small set of curated transform functions by applying the generalized Stone-Weierstrass theorem for locally compact spaces. We state that theorem here. As in \S\ref{sec:expressiveness}, let $C_0(\Rbb^n, \Rbb)$ be the space of real-valued continuous functions $\Rbb^n\to \Rbb$ that vanish at infinity.

\begin{theorem}[Generalized Stone–Weierstrass theorem \citep{Stone_Ross_2007}] Let $X$ be a locally compact Hausdorff space and $A$ be a subalgebra of $C_0(X, \Rbb)$. Then A is dense in $C_0(X, \Rbb)$, with the topology of uniform convergence, if and only if it separates points and vanishes nowhere.
\end{theorem}

The following lemma is used in the proof of Lemma~\ref{lem:basis-function}:
\begin{lemma}
The set of Gaussian functions on $\Rbb^n$ form an associative algebra $A$. $A$ vanishes nowhere and separates points.
\end{lemma}
\begin{proof}
The Gaussian functions on $\Rbb^n$ form an associative algebra over $\Rbb$, as the product or sum of two Gaussians on $\Rbb^n$ is a Gaussian on $\Rbb^n$.
The algebra $A$ does not vanish anywhere, as a Gaussian function can be defined at any point in $\Rbb^n$.

$A$ separates points, i.e. for any two points $x, y\in \Rbb^n$, there exists a Gaussian function $f$ s.t. $f(x)\ne f(y)$. This holds as we can let $f$ be a Gaussian centered at $x$, hence $f(x)\ne f(y)$.
\end{proof}
%'TemporalLag(t1, 2)', 'poly_transform(x1)', 'poly_transform(x2)', 'RelTemporalMean(series0)', 'quantilize(x3)', 'TemporalLag(t1, 3)', 'TemporalDiff(t0, t1, t2, t3, t4; 1)', 'TemporalMean(series1)'